\documentclass{easychair}

\usepackage{doc}

\usepackage{graphicx}
\usepackage{subcaption}
\usepackage{amsmath}
\usepackage{amssymb,amsfonts,textcomp}
\usepackage{amsthm}
\usepackage{stmaryrd}
\usepackage[ruled, linesnumbered, vlined]{algorithm2e}
\usepackage{booktabs}
\usepackage{verbatim}

\title{Wayeb: a Tool for Complex Event Forecasting\footnote{Paper published at LPAR-22, 22nd International Conference on Logic for Programming, Artificial Intelligence and Reasoning, Awassa Ethiopia, 16-21 November 2018, \url{https://easychair.org/publications/paper/VKP1}.}}

\author{
Elias Alevizos\inst{1}\inst{3}
\and
Alexander Artikis\inst{2}\inst{1}
\and
Georgios Paliouras\inst{1}
}

\institute{
  NCSR Demokritos,
  Athens, Greece\\
  \email{ \{alevizos.elias,a.artikis,paliourg\}@iit.demokritos.gr }
\and
   University of Piraeus,
   Piraeus, Greece
\and
   National and Kapodistrian University of Athens,
   Athens, Greece
}

\authorrunning{Alevizos, Artikis, Paliouras}

\titlerunning{Wayeb}

\def\dfasr{$\mathit{DFA_{\Sigma^{*}\cdot R}}$}

\def\pmcmr{$\mathit{PMC_{R}^{m}}$}

\def\pfc{$\mathit{\theta_{fc}}$}
\def\trueb{\textsf{\footnotesize TRUE}}
\def\sfa{$\mathit{SFA}$}
\def\where{\textsf{\footnotesize WHERE}}
\def\partition{\textsf{\footnotesize PARTITION BY}}
\def\and{\textsf{\footnotesize AND}}

\newtheorem{definition}{Definition}

\newtheorem{proposition}{Proposition}

\begin{document}

\maketitle

\begin{abstract}
Complex Event Processing (CEP) systems have appeared in abundance during the last two decades.
Their purpose is to detect in real--time interesting patterns upon a stream of events and to inform an analyst for the occurrence of such patterns in a timely manner.
However, there is a lack of methods for forecasting when a pattern might occur before such an occurrence is actually detected by a CEP engine.
We present Wayeb, a tool that attempts to address the issue of Complex Event Forecasting.
Wayeb employs symbolic automata as a computational model for pattern detection and Markov chains for deriving a probabilistic description of a symbolic automaton. 
\end{abstract}

%

%
%

\section{Introduction}

A Complex Event Processing (CEP) system takes as input a stream of events,
along with a set of patterns,
defining relations among the input events,
and detects instances of pattern satisfaction,
thus producing an output stream of complex events .
Typically, an event has the structure of a tuple of values which might be numerical or categorical,
with the \emph{event type} and \emph{timestamp} being the most common attributes.
Since time is of critical importance for CEP,
a temporal formalism is used in order to define the patterns to be detected.
Such a pattern imposes temporal (and possibly atemporal) constraints on the input events,
which, if satisfied, lead to the detection of a complex event.
Efficient processing is of paramount importance since complex events must be detected with very strict latency requirements.
Moreover, both the input events and the patterns may carry a certain degree of uncertainty,
e.g., due to noisy sensors or an incomplete knowledge of the domain.
These issues have been the focus of CEP research for the past two decades. 
See \cite{cugola_processing_2012,alevizos2017probabilistic} for surveys of (probabilistic) CEP. 
In this paper,
we focus on another dimension of CEP research,
that of Complex Event Forecasting (CEF),
i.e.,
the ability of a CEP system to provide forecasts about the possible occurrence of complex events in the future.
Although some conceptual proposals towards this direction have appeared in the literature \cite{fulop_predictive_2012,engel_towards_2011},
there still remains a lack of concrete algorithms and systems that can actually perform CEF.
In what follows,
we present Wayeb,
a CEP/F engine that,
besides detecting CEP patterns,
it can also forecast when such patterns may occur.

\section{Forecasting with Classical Automata}
\label{sec:classical}

In order to allow for a self--contained presentation,
we begin by briefly describing how our proposed method works with classical automata
(for details, see \cite{alevizos2017event}).
In this case,
we assume that the stream $S=t_{1},t_{2},\cdots$ is sequence of symbols from a finite alphabet $\Sigma=\{e_{1},...,e_{r}\}$,
i.e., $t_{i} \in \Sigma$.  
A pattern is defined as a (classical) regular expression $R$ and the usual tools of standard automata theory may be employed \cite{hopcroft_introduction_2007}.
The finite automaton used for event detection is the one corresponding to the expression $\Sigma^{*} \cdot R$
($\cdot$ denotes concatenation and $^{*}$ Kleene--star),
since it should work on streams and be able to start the detection at any point.
Appending $\Sigma^{*}$ at the beginning allows the automaton to skip any number of events.
As a next step,
we construct the deterministic finite automaton (DFA) for $\Sigma^{*} \cdot R$, \dfasr.
This allows us to convert \dfasr\ to a Markov chain.
If we assume that the stream is composed of i.i.d. events from $\Sigma$,
then the sequence $Y=Y_{0},Y_{1},...,Y_{i},...$, where $Y_{0}=q_{0}$ and $Y_{i}=\delta(Y_{i-1},t_{i})$,
with $q_{0}$ the automaton's start state and $\delta$ its transition function, 
is a 1-order Markov chain \cite{nuel_pattern_2008}.
Such a Markov chain, associated with a pattern $R$, is called a Pattern Markov Chain (PMC).
The transition probability between two states connected with the symbol $e_{j}$ is simply its occurrence probability, $P(X_{i}{=}e_{j})$. 
If we assume that the process generating the stream is of a higher order
$m \geq 1$,
then we must first convert \dfasr\ to an \emph{$m$--unambiguous} DFA,
i.e.,
to an equivalent DFA where each state can ``remember'' the last $m$ symbols.
This can be achieved by iteratively duplicating those states of \dfasr$\ $
for which we cannot unambiguously determine the last $m$ symbols that can lead to them 
and then convert it to a PMC, 
denoted by \pmcmr \cite{nicodeme_motif_2002,nuel_pattern_2008}.

After constructing \pmcmr,
we can estimate its waiting-time distributions. 
The waiting-time $W_{R}(q)$ for each non--final state $q$ of \dfasr\ is a random variable,
defined as the number of transitions until \dfasr\ visits for the first time one of its final states:
$\mathit{W_{R}(q){=}inf\{n{:}Y_{0},Y_{1},...,Y_{n}, Y_{0}{=}q, q {\in} Q \backslash F, Y_{n} {\in} F\}}$.
We compute the distribution of $W_{R}(q)$ by converting each state of \pmcmr\ that corresponds to 
a final state of the DFA into an absorbing state
and then re--organizing the transition matrix as follows:
$\boldsymbol{\Pi} = 
\begin{pmatrix} 
\boldsymbol{N_{(l-k) \times (l-k)}} & \boldsymbol{C_{(l-k) \times k}}  \\ 
\boldsymbol{0_{k \times (l-k)}} & \boldsymbol{I_{k \times k}}
\end{pmatrix}$,
assuming \pmcmr\  has a total of $l$ states,
$k$ of which are final.
$\boldsymbol{N}$ holds the probabilities for all the possible transitions between (and only between) the
non-final states,
whereas $\boldsymbol{C}$ holds the transition probabilities from non-final to final states.
Then, the probability for the time index $n$ when the system first enters the set of absorbing states is given by \cite{fu_distribution_2003}:
$P(Y_{n} \in A, Y_{n-1} \notin A,...,Y_{1} \notin A \mid \boldsymbol{\xi_{init}}) =
\boldsymbol{\xi}^{T}\boldsymbol{N}^{n-1}(\boldsymbol{I}-\boldsymbol{N})\boldsymbol{1}$, 
where $A$ denotes the set of absorbing states,
$\boldsymbol{\xi_{init}}$ is the initial distribution on the states
and $\boldsymbol{\xi}$ consists of the $l-k$ elements of $\boldsymbol{\xi_{init}}$ corresponding to non-absorbing states.
Since, in our case, the current state of \dfasr\ is known, 
the vector $\boldsymbol{\xi_{init}^{T}}$ would have  $1.0$ as the value for the element corresponding to the current state (and 0 elsewhere).
$\boldsymbol{\xi}$ changes dynamically as the DFA/PMC moves among its various states and every
state has its own $\boldsymbol{\xi}$, denoted by $\boldsymbol{\xi_{q}}$:
\[ \boldsymbol{\xi_{q}}(i) =
  \begin{cases}
    1.0 & \quad \text{if row } i \text{ of } \boldsymbol{N} \text{ corresponds to state } q   \\
    0  & \quad \text{otherwise} \\
  \end{cases}
\] 
The probability of $W_{R}(q)$ is then given by:
$P(W_{R}(q)=n)=\boldsymbol{\xi_{q}}^{T}\boldsymbol{N}^{n-1}(\boldsymbol{I}-\boldsymbol{N})\boldsymbol{1}$

\begin{figure}[t]
    \centering
    \begin{subfigure}[b]{0.38\textwidth}
        \includegraphics[width=\textwidth]{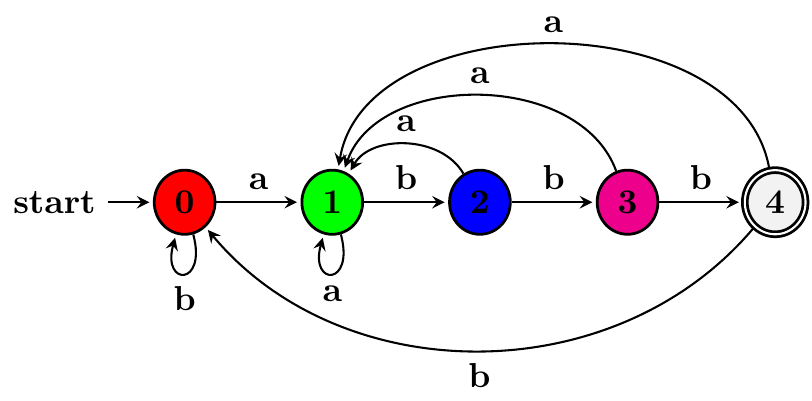}
        \caption{DFA.}\label{fig:dfaabbb}
    \end{subfigure}
    \begin{subfigure}[b]{0.43\textwidth}
        \includegraphics[width=\textwidth]{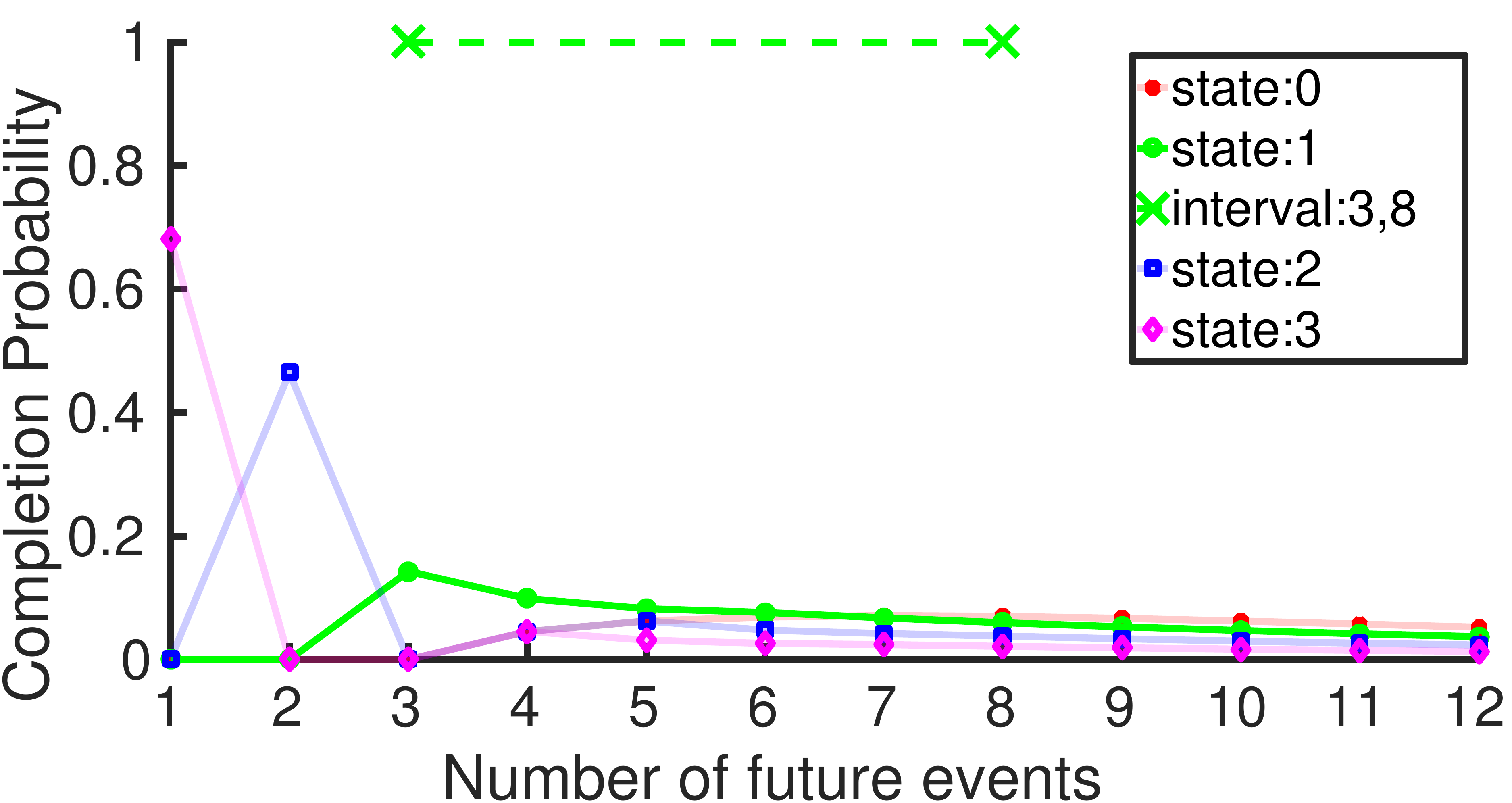}
        \caption{Waiting-time distributions and smallest interval exceeding $\theta_{fc}=50\%$ for state 1.}\label{fig:wt1}
    \end{subfigure}
    \caption{DFA and waiting-time distributions for $R=a\cdot b\cdot b\cdot b$, $\Sigma=\{a,b\}$, $m=0$.}\label{fig:wtdfas}
\end{figure}

The transition matrix $\Pi$ is learnt using the maximum-likelihood estimators for its elements \cite{anderson_statistical_1957}:
$\hat{\pi}_{i,j}=\frac{n_{i,j}}{\sum_{k \in Q} n_{i,k}}=\frac{n_{i,j}}{n_{i}}$
where $n_{i}$ denotes the number of visits to state $i$,
$n_{i,j}$ the number of transitions from state $i$ to state $j$
and $\pi_{i,j}$ the transition probability from state $i$ to state $j$. 
After estimating the transition matrix,
we compute the waiting-time distributions and then build the forecasts associated with each state.
Each forecast is an interval $I=(start,end)$ and its meaning is the following:
given that the DFA is in a certain state, 
we forecast that it will have reached one of its final states at some future point between $start$ and $end$,
with probability at least \pfc.
The calculation of this interval is done by using the waiting-time distribution that corresponds to each state
and the threshold \pfc\ is set beforehand by the user.
We use a single-pass algorithm in order to scan the distribution of each state and find the \emph{smallest} interval whose probability, 
i.e.,
the sum of probabilities of points included in the interval, 
exceeds \pfc. 
In this paper,
we assume stationarity,
thus waiting--time distributions are computed once,
after matrix estimation.
An example of how forecasts are produced is shown in Figure \ref{fig:wtdfas}.
Figure \ref{fig:dfaabbb} shows a simple DFA that can detect the pattern $R=a\cdot b\cdot b\cdot b$ on a stream composed of $a$ and $b$ events.
For its four non-final states,
Figure \ref{fig:wt1} shows their corresponding waiting--time distributions. 
For state $1$, 
this figure also shows the forecast interval produced by scanning the corresponding (green and highlighted) distribution,
when $\theta_{fc}=50\%$. 
If the DFA moves to another state,
then another distribution will be activated and a different forecast interval will be produced
(these other intervals are not shown to avoid cluttering).
\section{Forecasting with Symbolic Automata}
\label{sec:sumbolic_auto}

One limitation of the method presented above is that it requires a finite alphabet,
i.e.,
the input stream may be composed only of a finite set of symbols;
on the other hand, events in a stream are in the form of tuples and look more like data words.
Their attributes might be real--valued, 
taking values from an infinite set, 
and thus cannot be directly handled by DFA.
In order to overcome this limitation,
we have extended our method so that symbolic automata are employed \cite{veanes_rex_2010,dantoni_power_2017}.
Symbolic automata,
instead of having symbols from a finite set on their transitions,
are equipped with predicates from a Boolean algebra,
acting as transition guards.
Such predicates can reference any event attribute,
thus allowing for more expressive patterns.

We now present a formal definition of symbolic automata (for details, see \cite{dantoni_power_2017}).
\begin{definition}[Symbolic Automaton]
A symbolic finite automaton (\sfa) is a tuple \linebreak $M=$($\mathcal{A}$,$Q$,$q^{0}$,$F$,$\Delta$), 
where $\mathcal{A}$ is an effective Boolean algebra, $Q$ is a finite set of states, $q^{0} \in Q$ is the initial state, $F \subseteq Q$ is the set of final states and $\Delta \subseteq Q \times \Psi_{\mathcal{A}} \times Q$ is a finite set of transitions, with $\Psi_{\mathcal{A}}$ being s set of predicates closed under the Boolean connectives.
\end{definition}
Most other definitions from classical automata carry over symbolic automata.
Importantly, 
\sfa\ are determinizable \cite{dantoni_power_2017}.
The definition for deterministic \sfa\ is also similar to that for classical automata,
with the important difference that it is not enough for all transitions from a state to have different predicates.
We require that predicates on transitions from the same state are mutually exclusive, 
i.e.,
at most one may evaluate to \trueb\ (see again \cite{dantoni_power_2017}).
The determinization process for \sfa\ is similar to that for classical automata and is based on the construction of the power--set of the states of the non--deterministic automaton.
Due to this issue of different predicates possibly both evaluating to \trueb\ for the same event/tuple,
we first need to create the \emph{minterms} of the predicates of a \sfa,
i.e.,
the set of maximal satisfiable Boolean combinations of  such  predicates.
When constructing a deterministic \sfa,
these minterms are used as guards on the transitions,
since they are mutually exclusive.

This result about \sfa\ being determinizable allows us to use same technique of converting a deterministic automaton to an $m$--unambiguous automaton and then to a Markov chain.
First, note that the set of minterms constructed from the predicates appearing in a \sfa\ $M$,
denoted by $\mathit{Minterms}(\mathit{Predicates}(M))$,
induces a finite set of equivalence classes on the (possibly infinite) set of domain elements of $M$ \cite{dantoni_power_2017}. 
For example, if $\mathit{Predicates}(M)=\{\psi_{1},\psi_{2}\}$,
then $\mathit{Minterms}(\mathit{Predicates}(M))=\{\psi_{1} \wedge \psi_{2}, \psi_{1} \wedge \neg \psi_{2}, \neg \psi_{1} \wedge \psi_{2}, \neg \psi_{1} \wedge \neg \psi_{2}\}$ and we can map each domain element, which, in our case, is an event/tuple, to exactly one of these 4 minterms:
the one that evaluates to \trueb\ when applied to the element.
Similarly, the set of minterms induces a set of equivalence classes on the set strings (event streams in our case).
For example, if $S{=}t_{1},\cdots,t_{n}$ is an event stream, 
then it could be mapped to $S^{'}{=}a,\cdots,b$,
with $a$ corresponding to $\psi_{1} \wedge \neg \psi_{2}$ if $\psi_{1}(t_{1}) \wedge \neg \psi_{2}(t_{1}) = \trueb$, $b$ to $\psi_{1} \wedge \psi_{2}$, etc.
We say that $S^{'}$ is the stream induced by applying $\mathit{Minterms}(\mathit{Predicates}(M))$ on the original stream $S$.
We first give a definition for an an $m$--unambiguous deterministic \sfa,
by modifying the relevant definition for classical automata \cite{nuel_pattern_2008}:
\begin{definition}[$m$--unambiguous deterministic \sfa]
A deterministic \sfa\ ($\mathcal{A}$,$Q$,$q^{0}$,$F$,$\Delta$) is $m$--ambiguous if there exist $q \in Q$ and $a,b \in T^{m}$ such that $a \neq b$ and $\delta(q,a)=\delta(q,b)$. 
A deterministic \sfa\ which is not $m$--ambiguous is $m$--unambiguous.
\end{definition}
In other words,
upon reaching a state $q$ of an $m$--unambiguous \sfa,
we know which last $m$ minterms evaluated to \trueb,
i.e.,
we know the last $m$ symbols of the induced stream $S^{'}$.
The following proposition then follows:
\begin{proposition}
Let $M$ be a deterministic $m$--unambiguous \sfa, 
$S=t_{1},\cdots,t_{n}$ an event stream 
and $S^{'}=X_{1},\cdots,X_{n}$ the stream induced by applying $T=\mathit{Minterms}(\mathit{Predicates}(M))$ on $S$. 
Assume that $S^{'}$ is a $m$--order Markov process, 
i.e.,
$P(X_{i}=a \mid X_{1},\cdots,X_{i-1}) = P(X_{i} = a \mid X_{i-m},\cdots,X_{i-1})$.
Then the sequence $Y=Y_{m},\cdots,Y_{n}$ defined by $Y_{0}=q^{0}$ and $Y_{i}=\delta(Y_{i-1},t_{i})$ 
(i.e., the sequence of states visited by $M$)
is a $1$--order Markov chain whose transition matrix is given by:
\[ \Pi(p,q) =
  \begin{cases}
    P(X_{m+1}=a \mid X_{1},\cdots,X_{m}=\phi(\delta^{-m}(p))) & \quad \text{if } \delta(p,a)=q   \\
    0  & \quad \text{if } q \notin \delta(p,T) \\
  \end{cases}
\]
where $\delta^{-m}(q) = \{a \in T^{m} \mid \exists p \in Q: \delta(p,a)=q\}$ is the set of concatenated labels of length $m$ that can lead to $q$ and, by definition, is a singleton for $m$--unambiguous \sfa. 
\end{proposition}
\begin{proof}
This result holds for classical $m$--unambiguous deterministic automata \cite{nuel_pattern_2008}.
For the symbolic case,
note that,
from an algebraic point of view,
the set $T=\mathit{Minterms}(\mathit{Predicates}(M))$ may be treated as a generator of the monoid $T^{*}$, 
with concatenation as the operation.
If the cardinality of $T$ is $k$,
then we can always find a set $\Sigma=\{a_{1},\cdots,a_{k}\}$ of $k$ distinct symbols and then a morphism (in fact, an isomorphism) $\phi: T^{*} \rightarrow \Sigma^{*}$ that maps each minterm to exactly one, unique $a_{i}$.
A classical deterministic automaton $N$ can then be constructed by relabelling the \sfa\ $M$ under $\phi$,
i.e.,
by copying/renaming the states and transitions of the original \sfa\ $M$ and by replacing the label of each transition of $M$ by the image of this label under $\phi$.
Then, the behavior of $N$ (the language it accepts) is the image under $\phi$ of the behavior of $M$ \cite{sakarovitch_elements_2009}.
Based on these observations, 
we can see that the sequence of states visited by an $m$--unambiguous deterministic \sfa\ is indeed a 1--order Markov chain,
as a direct consequence of the fact that a deterministic \sfa\ has the same behavior (up to isomorphism) to that of a classical deterministic automaton, constructed through relabelling.
\end{proof}

Therefore, 
after constructing a deterministic \sfa\,
we can use it to construct a PMC and learn its transition matrix.
The conditional probabilities in this case are essentially probabilities of seeing an event that will satisfy a predicate,
given that the previous event(s) have satisfied the same or other predicates.
For example,
if $m=1$ and $\mathit{Predicates}(M)=\{\psi_{1},\psi_{2}\}$, 
one such conditional probability would be $P(\psi_{1}(t_{i+1}) \wedge \psi_{2}(t_{i+1}) =  \trueb \mid \neg \psi_{1}(t_{i}) \wedge \neg \psi_{2}(t_{i}) = \trueb )$,
i.e.,
the probability of seeing an event that will satisfy both $\psi_{1}$ and $\psi_{2}$ given that the current, last seen event satisfies neither of these predicates.
As with classical automata,
we can then provide again forecasts based on the waiting--time distributions of the PMC derived from a \sfa.

\section{Experimental Results}
\label{sec:symbolic_experiments}
Wayeb is a Complex Event Forecasting engine based on symbolic automata, 
written in the Scala programming language.
It was tested against two real--world datasets coming from the field of maritime monitoring.
When sailing at sea, (most) vessels emit messages relaying information about their position, heading, speed, etc.: the so-called AIS (automatic identification system) messages. 
Such a stream of AIS messages can then be used in order to detect interesting patterns in the behavior of vessels \cite{patroumpas2017online}.
Two AIS datasets were used, made available in the datAcron project\footnote{\href{http://datacron-project.eu/}{http://datacron-project.eu/}}: 
the first contains AIS kinematic messages from vessels sailing in the Atlantic Ocean around the port of Brest, France, and span a period from 1 October 2015 to 31 March 2016 \cite{ray_cyril_2018_1167595}; 
the second was provided by IMISG\footnote{\href{https://imisglobal.com/}{https://imisglobal.com/}} and contains AIS kinematic messages from most of Europe 
(the entire Mediterranean Sea and parts of the Atlantic Ocean and the Baltic Sea), 
spanning an one--month period from 1 January 2016 to 31 January 2016.
AIS messages can be noisy, redundant and typically arrive at unspecified time intervals.
We first processed our datasets in order to produce clean and compressed trajectories, 
consisting of critical points,
i.e., important points that are a summary of the initial trajectory, 
but allow for an accurate reconstruction \cite{patroumpas2017online}.
Subsequently,
we sampled the compressed trajectories by interpolating between critical points
in order to get trajectories where each point has a temporal distance of one minute from its previous point.
After excluding points with null speed (in order to remove stopped vessels),
the final streams consist of $\approx$ 1.3 million points for the Brest dataset and $\approx$ 2.4 million points for the IMISG dataset.
The experiments were run on a machine with Intel Core i7-4770 CPU @ 3.40GHz processors and 16 GB of memory.

We have chosen to demonstrate our forecasting engine on two important patterns.
The first concerns a movement pattern in which vessels approach a port and the goal is to forecast when a vessel will enter the port.
This is a key target of several vessel tracking software platforms, 
as indicated by the 2018 challenge of the International Conference on Distributed and Event-Based Systems \cite{debs18challenge}. 
The second pattern concerns a fishing maneuver inside a fishing area.
Forecasting when vessels are about to start fishing could be important in order to manage the pressure exerted on fishing areas \cite{datacron51}.

The \emph{approaching} pattern may be defined as follows:
\begin{equation}
\label{pattern:approaching}
\begin{aligned}
R_{1} := & 	x \cdot y^{+} \cdot z \ \where \\
			& 	\mathit{Distance(x,\mathit{PortCoords},7.0,10.0)}\ \and\ \mathit{Distance(y,\mathit{PortCoords},5.0,7.0)}\ \and \\
			& 	\mathit{WithinCircle(z,\mathit{PortCoords},5.0)} \ \partition\ \mathit{vesselId} 
\end{aligned}
\end{equation}
Concatenation is denoted by $\cdot$ and $^{+}$ stands for Kleene$+$.
We want to start detecting a vessel's movement whenever a vessel is between 7 and 10 km away from a specific port,
then it approaches the port, 
with its distance from it falling to the range of 5 to 7 km,  
stays in that range for 1 or more messages, 
and finally enters the port, 
defined as being inside a circle with a radius of 5 km around the port.
We have chosen the above syntax in order to make clear what the regular part of a pattern is (before the \where\ keyword) and what its logical part is (after \where),
but predicates can be placed directly in the regular expression.
Note also that the first argument of a predicate is the event upon which it is to be applied,
but we also allow for other arguments to be passed as constants. 
The \emph{partition contiguity} selection strategy is employed \cite{zhang_complexity_2014},
i.e.,
for each new vessel appearing in the stream,
a new automaton run is created,
being responsible for this vessel.
This is just a special case of parametric trace slicing typically used in runtime verification tools \cite{chen2009parametric}.
If we use this pattern to build a PMC,
the transition probabilities will involve the three predicates that appear in it.
However, it is reasonable to assume that other features of a vessel's kinematic behavior could affect the accuracy of forecasts, 
e.g.,
its speed.
For this reason, 
we have also added a mechanism to our module that can incorporate in the PMC some extra features,
declared by the user,
but not present in the pattern itself.
The extra features we decided to add concern the vessel's speed and its heading:
$\mathit{SpeedBetween(x,0,10.0)}$, $\mathit{SpeedBetween(x,10.0,20.0)}$, $\mathit{SpeedBetween(x,20.0,30.0)}$ and $\mathit{HeadingTowards(x,PortCoords)}$.
The first three try to use the speed level of a vessel,
whereas the last one uses the vessel's heading and checks whether it is headed towards the port.
Our experiments were conducted using the main Brest port as the port of reference.

The \emph{fishing} pattern may be defined as follows:
\begin{equation}
\label{pattern:fishing}
\begin{aligned}
R_{2} := & x \cdot y^{*} \cdot z \ \where \\
			& 	(\mathit{IsFishingVessel(x)} \wedge \neg \mathit{InArea(x,FishingArea)}) \ \and \\ 
			&   (\mathit{InArea(y,FishingArea)} \wedge \mathit{SpeedBetween(y,9.0,20.0)}) \ \and \\
			& 	(\mathit{InArea(z,FishingArea)} \wedge \mathit{SpeedBetween(z,1.0,9.0)}) \\
			& 	\partition\ \mathit{vesselId}
\end{aligned}
\end{equation}
This definition attempts to capture a movement of a fishing vessel,
in which it is initially outside a specific fishing area,
then enters the area with a traveling speed (between 9 to 20 knots), 
remains there for zero or more messages,
and finally starts moving with a fishing speed (between 1 to 9 knots) while still in the same area.
We also used the same extra features as in Pattern \ref{pattern:approaching},
with the difference that the $\mathit{Heading}$ feature now concerns the area.
Note that by partitioning by $\mathit{vesselId}$ we don't have to add the $\mathit{IsFishingVessel}$ predicate for $y$ and $z$.

\begin{figure}[t]
    \centering
    \begin{subfigure}[b]{0.49\textwidth}
        \includegraphics[width=\textwidth]{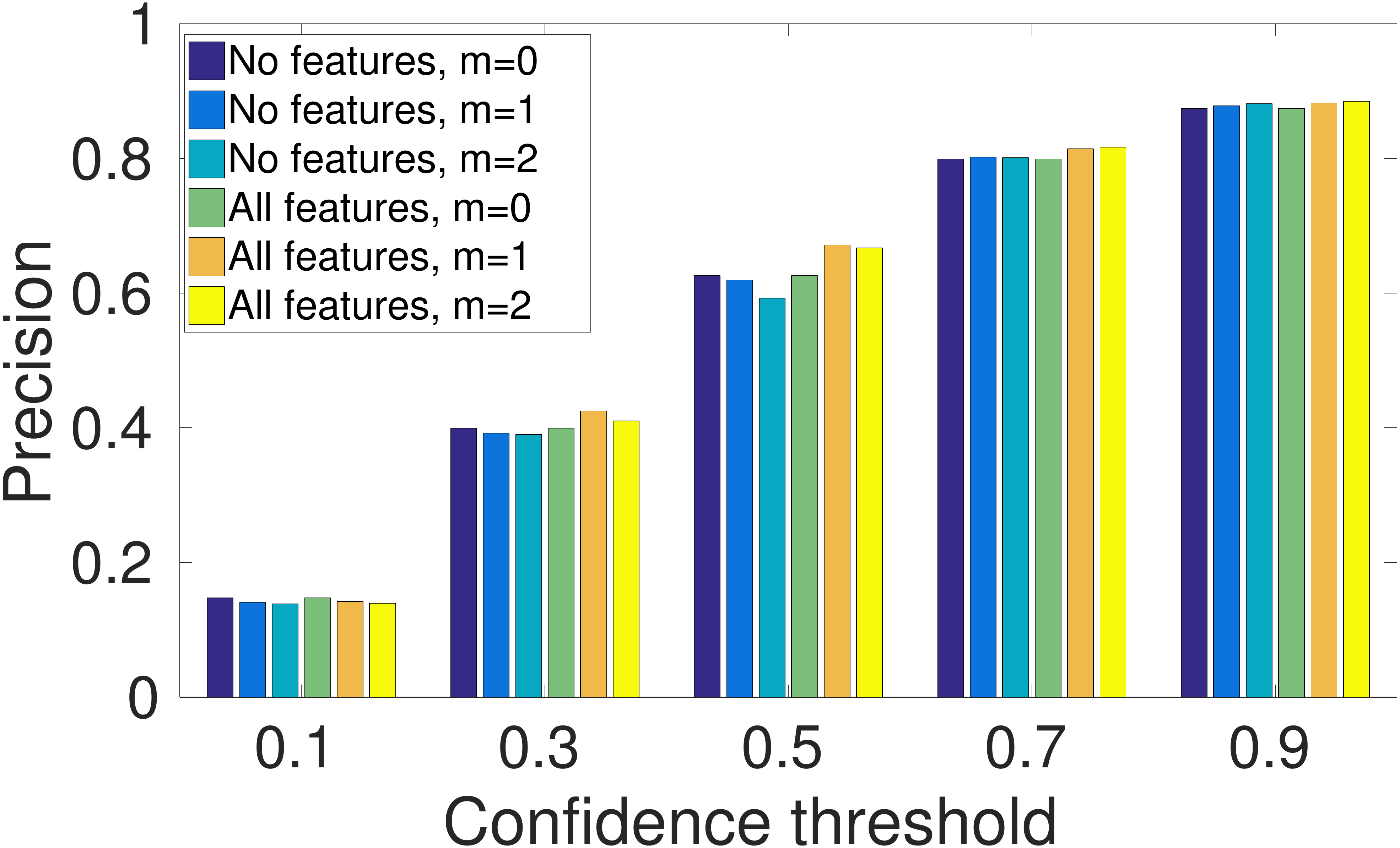}
        \caption{Precision (approaching pattern).}\label{fig:symbolic:port_prec}
    \end{subfigure}
    \begin{subfigure}[b]{0.49\textwidth}
        \includegraphics[width=\textwidth]{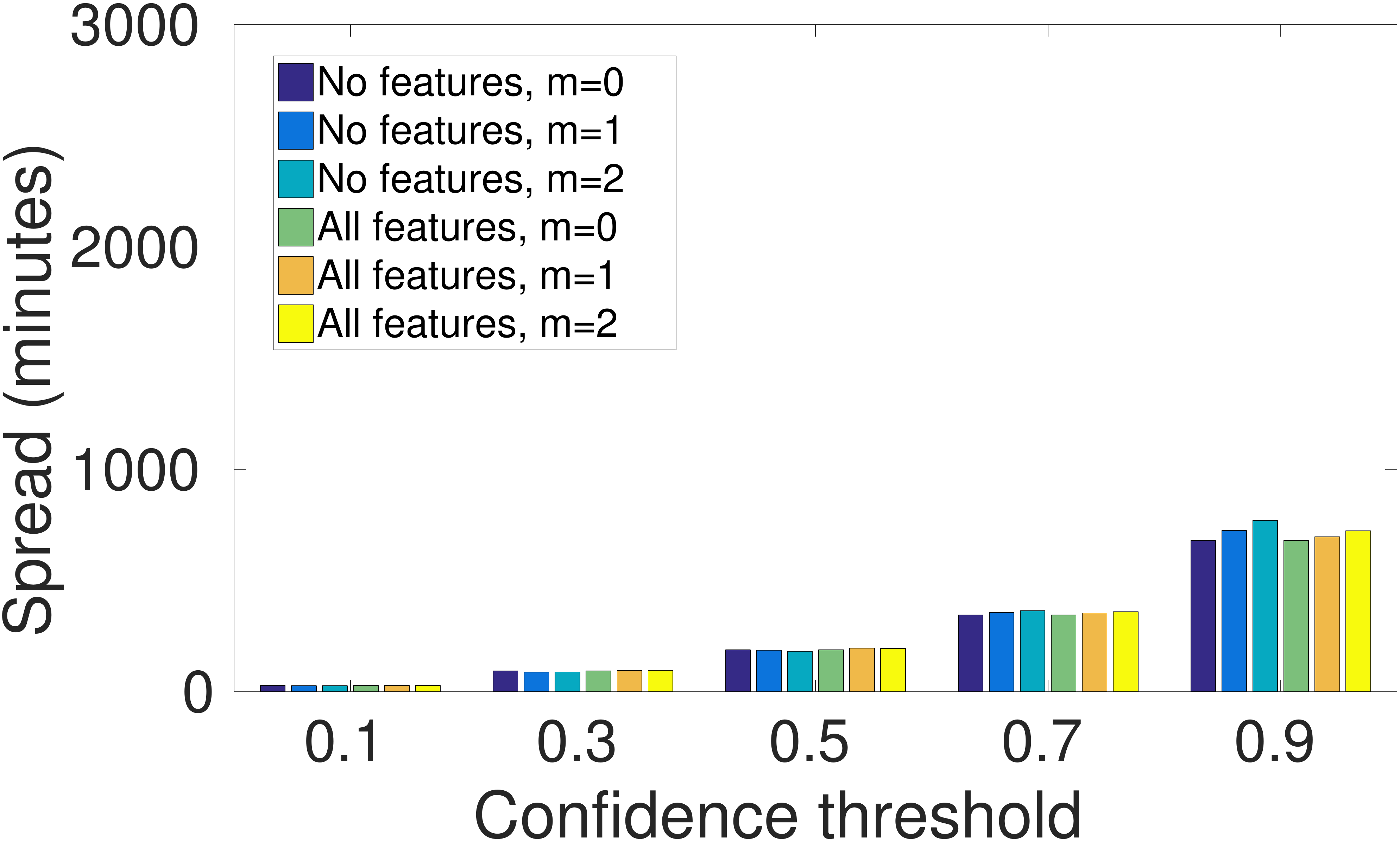}
        \caption{Spread (approaching pattern).}\label{fig:symbolic:port_spread}
    \end{subfigure}
    \hfill
    \begin{subfigure}[b]{0.49\textwidth}
        \includegraphics[width=\textwidth]{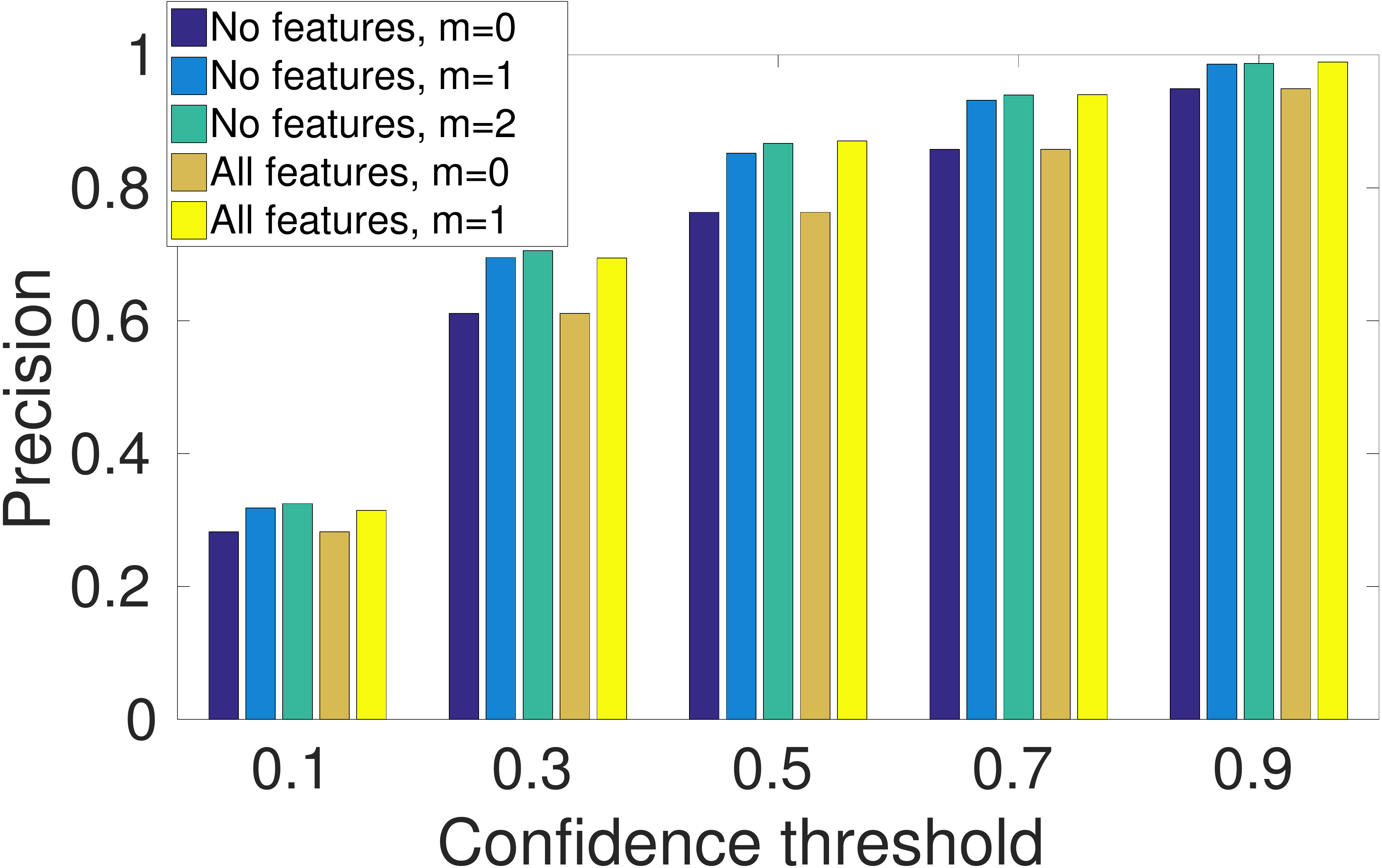}
        \caption{Precision (fishing pattern).}\label{fig:symbolic:fish_prec}
    \end{subfigure}
    \begin{subfigure}[b]{0.49\textwidth}
        \includegraphics[width=\textwidth]{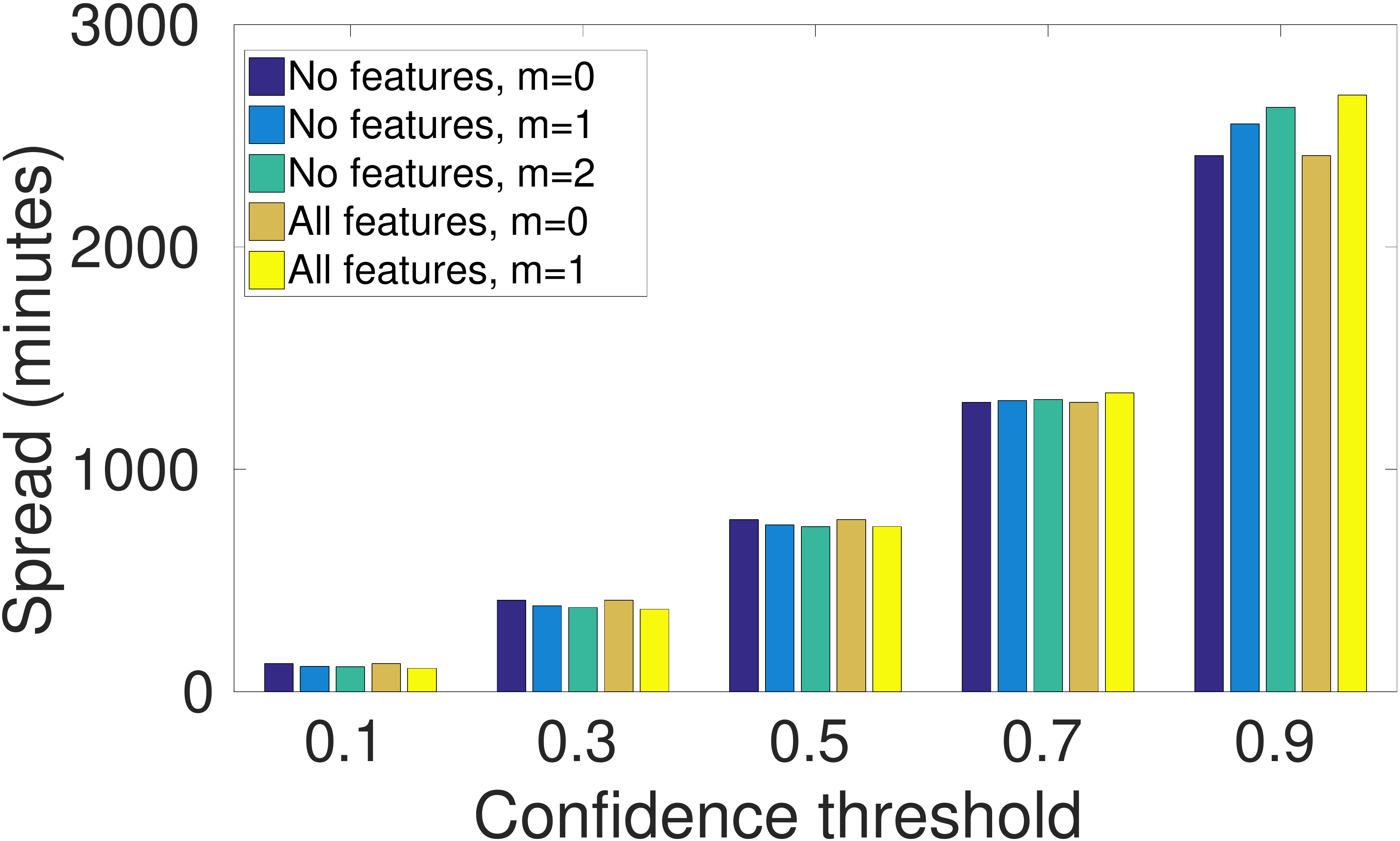}
        \caption{Spread (fishing pattern).}\label{fig:symbolic:fish_spread}
    \end{subfigure}
    \caption{Results for approaching the Brest port and for fishing.}\label{fig:symbolic:port_fish}
\end{figure}

Figure \ref{fig:symbolic:port_fish} shows
\emph{precision} and \emph{spread} results on the Brest dataset for different values of the forecasting confidence threshold $\theta_{fc}$
and of the assumed order $m$.
Precision is defined as the percentage of forecasts that were correct,
i.e.,
whose interval includes the timestamp of the complex event's actual occurrence.
Spread is defined as the length of a forecast interval.
The smaller the spread, the more focused a forecast is and thus more valuable.
Good results are therefore considered those with high precision scores (ideally 1.0) and low spread scores (ideally 0, i.e., single point forecasts).
For each confidence threshold,
each bar represents a different variation of the pattern.
We vary the order $m$ of the PMC in order to investigate the possible gains of looking deeper into the past.
There are also variations where no extra features are used (i.e., only the predicates of the original pattern are included in the PMC) and where all the extra features are included.
Our engine can always achieve precision scores that are above the confidence threshold.
However, note that,
as the confidence threshold increases,
the spread also tends to increase.
This happens because the engine tries to satisfy the $\theta_{fc}$ constraint by expanding the forecast intervals.
Looking back at the example of Figure \ref{fig:wt1},
if, instead of $\theta_{fc}=50\%$,
we set $\theta_{fc}=70\%$,
then the green interval shown at the top would have to be expanded to the right so that its probability exceeds the new increased threshold (expanding to the left would be pointless since only points with zero probability exist in this region).
It is interesting to also note that,
when we include the extra features,
the spread is lower for $\theta_{fc}=0.9$,
indicating that these features help in producing more focused forecasts.
Although a similar pattern can be observed for the \emph{fishing} pattern as well,
a more careful examination reveals that this pattern is more challenging.
The precision scores are higher,
but the spread is also significantly higher for all values of $\theta_{fc}$.
This means that accurately pinpointing when a fishing pattern will be detected is more difficult.
This result is also expected, 
since fishing maneuvers (often occurring in the open sea) exhibit a greater degree of variability,
whereas movement patterns while approaching a port are more or less straightforward.
The trade-off between precision and spread is thus always present,
but its exact nature also depends on both the pattern itself and the included features.
We also see that increasing $m$ does not seem to have a significant impact,
at least, with the predicates chosen here.
Note that this is not always the case,
since $m$ can indeed play a significant role in other domains and/or patterns \cite{alevizos2017event}.
As a general comment,
Figure \ref{fig:symbolic:port_fish} can help a user determine a satisfactory set of parameter values.
For example, for the approaching pattern,
a user could choose to set $\theta_{fc}=50\%$ and $m=0$,
which gives a high enough precision with relatively low spread and avoids the cost of disambiguation that accompanies any higher values of $m$.

In order to assess Wayeb's throughput, 
we run another series of experiments.
We tested the \emph{approaching} pattern on both the Brest and the Europe datasets.
For the Brest dataset,
all of the 222 ports of Brittany were included,
each with its own \sfa\ and PMC.
For the Europe dataset, 
222 European ports were randomly selected.
We tested for throughput both when forecasting is performed and when it is disabled,
in which case only recognition is performed.
The results are shown in Figure \ref{fig:symbolic:throughputs}.
As expected, throughput is lower when forecasting is enabled. 
However, the overhead is not significant. 
We also see that throughput is higher for the Europe dataset.
It is possible to achieve higher throughput because the event rate of the input stream is also higher.
The throughput is higher in this case,
despite the fact that it contains almost ten times as many vessels and has a higher incoming event rate,
indicating that Wayeb can scale well as the number of monitored objects increases.

\begin{figure}[t]
    \centering
    \includegraphics[width=0.6\textwidth]{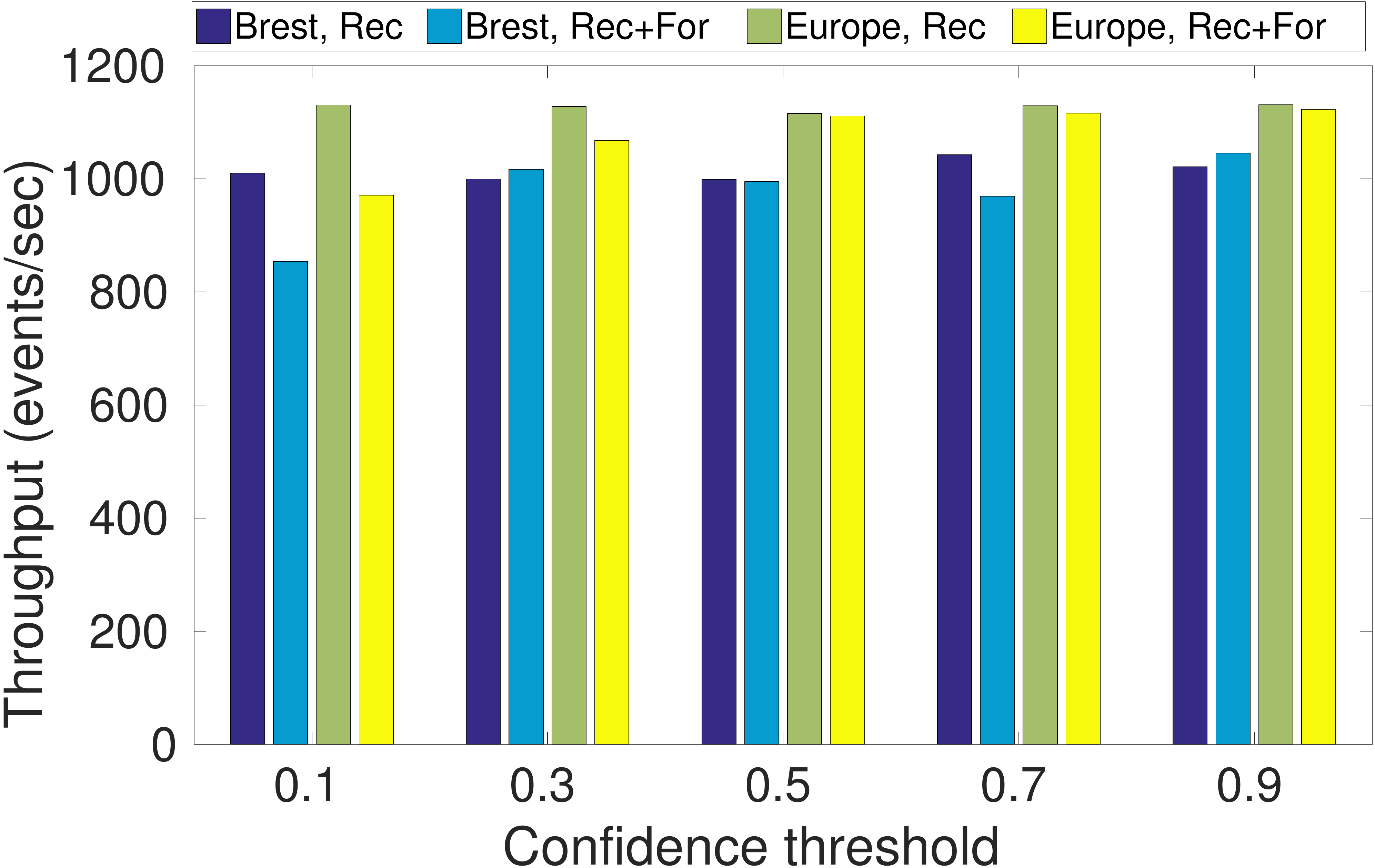}
    \caption{Throughput results. ``Rec'' denotes recongition only, with forecasting disabled. ``Rec+For'' denotes that both recognition and forecasting are enabled.}\label{fig:symbolic:throughputs}
\end{figure}

\section{Discussion}

Wayeb is one of the few CEP systems capable of the form of forecasting presented in this paper.
The only other system with similar capabilities is presented in \cite{muthusamy_predictive_2010},
using automata and Markov chains as well. 
The advantage of our method is that it allows for a deeper investigation of the past,
if this is needed, 
by being able to handle higher order Markov processes.
By using symbolic automata,
our computational model also has nice compositional properties
(a feature generally lacking in CEP)
and can accommodate any regular expression, 
without being restricted to sequential patterns.
Having predicates on the transitions also allows for an easy incorporation of Boolean expressions and of background knowledge.
For example, 
note that the $\mathit{IsFishingVessel(x)}$ predicate in Pattern \ref{pattern:fishing} is evaluated not by using information from the streaming AIS messages (which may not always contain correct information about a vessel's type) but by using the relevant background knowledge of the list of fishing vessels.
A significant number of forecasting methods comes from the field of temporal pattern mining,
e.g., \cite{vilalta_predicting_2002,laxman_stream_2008,zhou_pattern_2015}.
These are typically unsupervised methods with a focus on predicting what the next \emph{input} event(s) in a stream might be.
The same is true for the various sequence prediction methods \cite{begleiter2004prediction}. 
However, 
predicting the next input event(s) is not of the highest priority in CEP.
In fact, 
CEP engines typically employ \emph{selection strategies} that allow for ignoring input events that are not relevant for a given pattern.
Most input events might thus be irrelevant for a pattern and making predictions about them is not useful.
Instead, it is more important to forecast when a complex event will be detected,
as can be done with our method.
This is the reason why we cannot compare Wayeb to these methods,
since they have a different formulation of the forecasting task.
Moreover, sequence prediction methods,
as well as predictive CEP methods inspired by sequence prediction approaches \cite{gillani_pi-cep_2017},  
are based on the assumption of both a finite alphabet and a language of finite cardinality, 
basically excluding iteration.
Both of these assumptions do not hold in real-world applications and our method does not require them.

Although sequence prediction methods are not suitable for complex event forecasting as they are,
they employ techniques that could be useful in our case as well.
For example, 
due to the high cost of increasing the assumed order $m$ of the Markov process,
they employ variable--order Markov models \cite{buhlmann1999variable}.
We also intend to explore this direction,
in order to avoid the combinatorial explosion on the number of states of a PMC.
Another research direction is that of finding ways to compactly represent the past \cite{tino2001predicting} without having to enumerate every possible combination.
Finally,
the predicates of symbolic automata are unary and are applied only to the last event read from the stream.
This is a serious limitation for CEP where patterns are required having constraints between the last event and events seen in the past.
We intend to investigate if and how our method can be extended to other automata models that provide this functionality,
like extended symbolic automata \cite{dantoni_extended_2015} and quantified event automata \cite{barringer_quantified_2012}.

\section{Acknowledgments}
This work was supported by the EU H2020 datAcron project (grant agreement No 687591).

\label{sect:bib}
\bibliographystyle{abbrv}
\bibliography{refs}


\end{document}